\documentclass[10pt,conference]{IEEEtran}

\usepackage{booktabs}
\usepackage{mathrsfs}
\usepackage{textcomp, graphicx, subfigure}
\usepackage{epsfig}
\usepackage{amssymb}
\usepackage{amsmath}
\usepackage{amsfonts}
\usepackage{algorithmic}
\usepackage{algorithm}
\usepackage{multirow}
\usepackage{cite}
\usepackage{mathtools,bbm}
\usepackage{enumerate}
\usepackage{lipsum}
\usepackage{mathtools}
\usepackage{cuted}
\usepackage{cancel}
\usepackage{verbatim} 
\usepackage{mathtools}
 
 \usepackage{scalerel,amssymb}

\usepackage[utf8]{inputenc}
\usepackage[english]{babel}

\usepackage{amsthm}
\newtheorem{theorem}{Theorem}[section]
\usepackage{mathtools}
\DeclarePairedDelimiter\ceil{\lceil}{\rceil}
\begin{document}



\title{SGL: \underline{S}pectral \underline{G}raph \underline{L}earning from   Measurements
}

\author{\IEEEauthorblockN{Zhuo Feng}
\IEEEauthorblockA{\textit{Department of Electrical and Computer Engineering} \\
\textit{Stevens Institute of Technology}\\
Hoboken, NJ, USA \\
zhuo.feng@stevens.edu}

}
\IEEEoverridecommandlockouts
\IEEEpubid{\makebox[\columnwidth]{978-1-6654-3274-0/21/\$31.00~\copyright2021 IEEE \hfill} \hspace{\columnsep}\makebox[\columnwidth]{ }}
\maketitle
\begin{abstract}
This work   introduces a highly-scalable spectral graph densification framework  for learning resistor networks with linear measurements, such as node voltages and currents. We prove that given $O(\log N)$ pairs of voltage and current  measurements, it is possible to recover  ultra-sparse  $N$-node resistor networks which  can well preserve the effective resistance distances on the graph. In addition, the learned graphs also preserve the structural  (spectral) properties of the original graph, which  can potentially be leveraged in many circuit design and optimization tasks. We show that the proposed graph learning approach is equivalent to solving the classical graphical Lasso problems with Laplacian-like precision matrices.   
Through extensive experiments for a variety of real-world test cases, we show that  the proposed approach is highly scalable for learning ultra-sparse resistor networks without sacrificing solution quality.
\end{abstract}







\section{Introduction}\label{sect:introduction}
Recent years have seen a surge of interest in machine learning on graphs, with the goal of encoding high-dimensional data associated with nodes, edges, or (sub)graphs into low-dimensional vector representations that well preserve the original graph structural (manifold) information. Graph learning  techniques have shown promising results for various important applications such as vertex (data) classification, link prediction (recommendation systems), community detection, drug discovery, and electronic design automation (EDA) \cite{ wang2020gcn}. 

Modern graph learning  involves the following two key tasks: (1) graph topology learning for converting high-dimensional node feature (attribute) data into a  graph representation, and (2) graph embedding for converting graph-structured data (e.g. graph topology and node features) into low-dimensional vector representations. For example, an increasingly popular approach for analysing a data set in high-dimensional space (e.g. images of hand-written digits) is to first construct a graph connecting all the data points according to their similarities measured in distances in certain metric space \cite{dong2019learning}; next, graph embedding techniques are used to compute a low-dimensional vector representation of each data point (graph vertex), so that existing downstream machine learning or data mining algorithms can be conveniently applied.

However, {{even} the state-of-the-art  graph learning methods \cite{ egilmez2017graph, dong2019learning} do not scale comfortably to large data sets} due to their high algorithm complexity. For  example,  recent graph learning methods based on  Laplacian matrix estimation \cite{ egilmez2017graph, dong2019learning} have shown very promising performance. However, solving the required convex   optimization problem has a time complexity of {$O(N^2)$} per iteration for $N$ data points, which limits the application of these techniques to only very small data sets (e.g., with up to a few thousands data points). 

 For the first time, this paper introduces a spectral method for learning resistor networks from linear voltage and current measurements. Our approach is based on {a scalable spectral graph  densification algorithm} (SGL) for   estimation of   attractive Gaussian Markov Random Fields (GMRFs).
The proposed SGL algorithm  can efficiently solve  the graphical Lasso problem \cite{friedman2008sparse} with a Laplacian-like precision matrix  by iteratively including the most influential edges to dramatically reduce  spectral embedding distortions. A    unique property of the learned graphs  is that  {the spectral embedding or effective-resistance  distances on the constructed graph  will encode the similarities} between the original input data points (node voltage measurements).  To  achieve high efficiency, SGL exploits a scalable spectral graph   embedding scheme,  which allows each iteration to be completed in $O(N \log N)$ time, whereas existing state-of-the-art methods \cite{ egilmez2017graph,dong2019learning}    require at least $O(N^2)$ time for each iteration. Our analysis for sample complexity   shows that by leveraging the SGL algorithm  it is possible to  accurately estimate a sparse resistor network with only $O(\log N)$ voltage (and current) measurements (vectors).

The rest of this paper is organized as follows.    Section \ref{main_sec} introduces the proposed  spectral graph learning (SGL) framework is described in detail, which also includes the sample and algorithm complexity analysis.  Section
\ref{result_sec}  demonstrates extensive experimental results for learning a variety of real-world, large-scale graph problems, which is followed by the conclusion of this work in Section \ref{conclusion}.


 \section{SGL: A Spectral  Learning Approach}\label{main_sec}
 Suppose we are given $M$ linear measurements of  $N$-dimensional voltage and current vectors   stored in  data matrices $X\in {\mathbb{R} ^{N\times M}}$ and $Y\in {\mathbb{R} ^{N\times M}}$. The $i$-th column vector $X(:,i)$ corresponds to a  voltage response (graph signal) vector  due to the  $i$-th current excitation vector $Y(:,i)$. Motivated by recent graph learning research \cite{dong2019learning}, we propose a   scalable  spectral method (SGL) for    graph Laplacian   matrix  estimation  by exploiting the voltage ($X$) and current ($Y$) measurements.
 \subsection{Graph Learning via Laplacian Estimation}

Similar to the graphical Lasso problem \cite{friedman2008sparse}, the recent graph signal processing (GSP) based Laplacian estimation methods \cite{dong2019learning}  aim  to  learn  graph structures such that  graph signals will vary smoothly   across connected neighboring nodes  \cite{dong2019learning}.  To quantify the smoothness of a graph signal vector $x$ over a   undirected graph   $G=(V,E, w)$,  the following Laplacian quadratic form  can been adopted:
\begin{equation}\label{quad_form}
{x^\top}L x= \sum\limits_{\left( {s,t} \right) \in E}
{{w_{s,t}}{{\left( {x\left( s \right) - x\left( t \right)}
\right)}^2}},
\end{equation}
where $L=D-W$ denotes the graph Laplacian  matrix, $w_{s,t}=W(s,t)$  denotes  the weight for edge ($s,t$), while  $D$ and $W$ denote the  degree  and   the weighted adjacency matrices, respectively. The  GSP-based  graph learning  targets the following  convex optimization task \cite{dong2019learning}:
\begin{equation}\label{opt2}
  {\max_{\Theta}}:~ F= \log\det(\Theta)-  \frac{1}{M}Tr({X^\top \Theta X})-\beta {{\|\Theta\|}}^{}_{1},
\end{equation}
where   $\Theta={L}+\frac{I}{\sigma^2}$, ${L}$ denotes the set of valid graph Laplacian matrices, $Tr(\bullet)$ denotes the matrix trace, $I$ denotes the identity matrix, and  $\sigma^2>0$ denotes prior feature variance. In addition, $\|\bullet\|$ denotes the entry-wise $\ell_1$ norm, so ${\beta}{{\|\Theta\|}}^{}_{1} $ becomes the sparsity promoting regularization term. 
Since $\Theta={L}+\frac{I}{\sigma^2}$   is a symmetric and positive definite  (PSD) matrix (or M matrix) with non-positive off-diagonal entries, this formulation will lead to the estimation of  attractive GMRFs \cite{ dong2019learning, slawski2015estimation}.

\subsection{Gradient Estimation via  Sensitivity Analysis}\label{sec:theory}
We can express the graph Laplacian matrix as follows
\begin{equation}\label{LaplacianEdge}
L=\sum\limits_{\left( {s,t} \right) \in E}
{w_{s,t}}e_{s,t}e^\top_{s,t}
\end{equation}
where  ${e_{s}} \in \mathbb{R}^N$ denotes  the standard basis vector with all zero entries except for the $s$-th entry being $1$, and    ${e_{s,t}}=e_s-e_t$. Substituting (\ref{LaplacianEdge}) into   the  objective function $F$  in (\ref{opt2}),
and taking the partial derivative with respect to  $w_{s,t}$ leads to:
\begin{equation}\label{optF}
\frac{\partial F}{\partial w_{s,t} } = \sum\limits_{i=1}^N \frac{1}{\lambda_i+1/\sigma^2} \frac{\partial \lambda_i}{\partial w_{s,t} }
-  \frac{1}{M}\|X^\top e_{s,t}\|_2^2-4 \beta,
\end{equation}
where  the  eigenvectors corresponding to the ascending eigenvalues  ${\lambda _i}$ are denoted by ${u_i}$ for $i=1,...,N$, which satisfies:
\begin{equation}\label{eigen}
L u_i = \lambda_i u_i.
\end{equation}
Note that  the last two terms in (\ref{optF}) both imply   constraints on graph sparsity: including more edges will result in a greater trace $Tr({X^\top \Theta X})$. Consequently, we can safely  choose to set $\beta=0$ in the rest of this work, which will not impact the ranking of influential edges and thus the solution quality of the proposed SGL algorithm. 
\begin{theorem}\label{thm:pertub}
The spectral perturbation $\delta {\lambda _i} $ due to the inclusion of a candidate edge   $({s,t})$  can be  estimated by: 
\begin{equation}\label{sc}
 \delta {\lambda _i}=  \delta w_{s,t}\left( {{{u_i^\top}e_{s,t}}  } \right)^2.
\end{equation} 
\end{theorem}
\begin{proof} 
Consider the following spectral perturbation analysis:
\begin{equation}\label{formula_eig_perturb1}
\left( {L + \delta L} \right)\left( {{u_i} + \delta {u_i}} \right) = \left( {{\lambda _i} + \delta {\lambda _i}} \right)\left( {{u_i} + \delta {u_i}} \right),
\end{equation}
where a perturbation $\delta L$ that includes a new edge connection  is applied to $L$, resulting in perturbed eigenvalues and eigenvectors  ${\lambda _i} + \delta {\lambda _i}$ and ${u_i} + \delta {u_i}$ for $i=1,...,N$, respectively. Keeping only the first-order terms leads to:
\begin{equation}\label{formula_eig_perturb1_first_order}
 {L}\delta {u_i} + {\delta L}{u_i} = {{\lambda _i}{\delta {u_i}} + \delta {\lambda _i}}  {{u_i} }.
\end{equation}
Write $\delta u_i$ in terms of the original eigenvectors $u_i$ for for $i=1,...,N$:
\begin{equation}\label{delta u_i}
{\delta {u_i}} = \sum\limits_{i = 1}^{N} {{\alpha _i}{u_i}}.
\end{equation}
Substituting (\ref{delta u_i}) into (\ref{formula_eig_perturb1_first_order}) leads to:
\begin{equation}\label{formula_eig_perturb1_first_order_expand}
 {L}\sum\limits_{i = 1}^{N} {{\alpha _i}{u_i}} + {\delta L}{u_i} = {{\lambda _i}\sum\limits_{i = 1}^{N} {{\alpha _i}{u_i}} + \delta {\lambda _i}}  {{u_i} }.
\end{equation}
Multiplying ${u_i^\top}$ to both sides of (\ref{formula_eig_perturb1_first_order_expand}) leads to:
\begin{equation}\label{formula_eig_perturb1_conclusion}
 \delta {\lambda _i} = \delta w_{s,t}\left( {{{u_i^\top}e_{s,t}} } \right)^2.
\end{equation}
\end{proof}

For a connected graph, we  construct the following  subspace matrix for spectral graph embedding with the first $r-1$  nontrivial  Laplacian eigenvectors
\begin{equation}\label{subspace}
U_r=\left[\frac{u_2}{\sqrt {\lambda_2 +1/\sigma^2}},..., \frac{u_r}{\sqrt {\lambda_r +1/\sigma^2}}\right]. 
\end{equation}
According to  Theorem \ref{thm:pertub},  (\ref{optF}) can be approximated as: 
\begin{equation}\label{optF2}
 s_{s,t}=\frac{\partial F}{\partial w_{s,t} } \approx \|U_r^\top e_{s,t}\|_2^2
-  \frac{1}{M}\|X^\top e_{s,t}\|_2^2=z_{s,t}^{emb}-\frac{1}{M}z_{s,t}^{data},
\end{equation}
where $z_{s,t}^{emb}=\|U_r^\top e_{s,t}\|^2_2$ and $z_{s,t}^{data}={\|X^\top e_{s,t}\|^2_2}{}$ denote the $\ell_2$ distances in the spectral embedding space and the    data (voltage measurement) vector space, respectively.  The partial derivative ($ s_{s,t}$) in (\ref{optF2}) can   be regarded as each  edge's sensitivity that can be leveraged for solving the optimization task in (\ref{opt2}) using gradient based methods, such as the general stagewise algorithm for  group-structured learning \cite{tibshirani2015general}.

\subsection{Convergence Analysis of the SGL Algorithm}
If we define the {spectral embedding distortion} $\eta_{s,t}$ of an edge $(s, t)$  to be: 
 \begin{equation}\label{embedDist}
 \eta_{s,t}=M\frac{ z_{s,t}^{emb}}{z_{s,t}^{data}}.
 \end{equation}
Since $z_{s,t}^{emb}$ is equivalent to  the effective resistance $R^\textit{eff}_{s,t}$   on the graph  when $\sigma^2\rightarrow +\infty$   and  $r\rightarrow N$, we can rewrite the spectral embedding distortion as
 \begin{equation}\label{embedDist2}
 \eta_{s,t}=w_{s,t}R^\textit{eff}_{s,t},
 \end{equation}
where  ${w_{s,t}}=\frac{M}{z_{s,t}^{data}}$.  (\ref{embedDist2}) implies that $\eta_{s,t}$ becomes the edge leverage score for spectral graph sparsification \cite{spielman2011graph}. Prior work shows that    every undirected graph has  a nearly-linear-sized spectral sparsifier with $O (N\log N)$ edges which can be obtained by sampling each edge with a probability proportional to its edge leverage score \cite{spielman2011graph}; on the other hand, the  proposed  SGL graph learning framework can be regarded as a \textbf{spectral graph densification} procedure that aims to  construct a graph with   $O (N\log N)$  edges such that the   spectral embedding (effective-resistance) distances will   encode the $\ell_2$ distances between the original data points (voltage measurements).   The global (maximum) optimal solution of (\ref{opt2}) can be obtained when the maximum edge sensitivity ($s_{max}$) in (\ref{optF2}) becomes zero  or equivalently when the maximum  spectral embedding distortion  ($\eta_{max}$) in (\ref{embedDist})  becomes one. 

\subsection{Sample Complexity of the SGL Algorithm}
We analyze the required number of voltage vectors (measurements) for accurate graph learning via the SGL approach. Assume that $\sigma^2\rightarrow +\infty$. Denote the ground-truth graph by $G_*$,  and define its edge weight matrix  $W_*$ to be a diagonal matrix with $W_*(i,i)=w_i$, and  its injection matrix as:
\begin{equation}
B_*(i,p)=\begin{cases}
1 & \text{ if } p \text{ is  i-th edge's    head}\\
-1 & \text{ if } p \text{ is   i-th edge's tail} \\
0 & \text{ otherwise }. 
\end{cases}
\end{equation}  
Then the Laplacian matrix of the ground-truth graph $G_*$ in (\ref{LaplacianEdge}) can also be written as $L_*=B_*^\top W_* B_*$. Consequently, the effective resistance $R_*^\textit{eff}({s,t})$ between nodes $s$ and $t$ becomes:
\begin{equation}\label{effRes2}
R_*^\textit{eff}({s,t})=e^\top_{s,t}L_*^+e_{s,t}=\|W_*^{\frac{1}{2}} B_*L_*^{+}e_{s,t}\|_2^2,
\end{equation}
where $L^+_*$ denotes the Moore–Penrose pseudoinverse of $L_*$. According to the Johnson-Lindenstrauss Lemma, the effective-resistance distance for every pair of nodes satisfies \cite{spielman2011graph}:
\begin{equation}\label{effResJL}
(1-\epsilon)R_*^\textit{eff}({s,t})\le\|X^\top e_{s,t}\|_2^2 \le (1+\epsilon)R_*^\textit{eff}({s,t}),
\end{equation} 
where  the voltage measurement matrix $X\in \mathbb{R}^{N\times M}$ is constructed by going through the following steps:
\begin{enumerate}
\item Let $C$ be a random $\pm \frac{1}{\sqrt{M}}$ matrix of dimension $M\times |E|$, where $|E|$ denotes the number of edges and $M=24 \log \frac{N}{\epsilon^2}$ denotes the number of voltage measurements;
\item Obtain   $Y=CW_*^{\frac{1}{2}}B_*$, with the $i$-th row vector denoted by $y^\top_i$; 
\item Solve $L_* x_i= y_i$ for all rows in $C$ ($1 \le i\le M$), and construct $X$ using $x_i$ as its $i$-th column vector.
\end{enumerate}
Consequently, given  $M\ge O(\log {N}/{\epsilon^2})$ voltage vectors (measurements) obtained through the above procedure, a $(1\pm \epsilon)$-approximate effective-resistance distance can be computed by $\tilde R_*^\textit{eff}(s,t)=\|X^\top e_{s,t}\|_2^2$ for any pair of nodes $(s,t)$ in the original graph $G_*$. Consider the following close connection between  effective resistances and  spectral graph properties:
\begin{equation}\label{formula_Reff}
R^\textit{eff}_{s,t}=\|U_N^\top e_{s,t}\|^2_2,~\text{where~} U_N=\left[\frac{u_2}{\sqrt {\lambda_2}},..., \frac{u_N}{\sqrt {\lambda_N}}\right].
\end{equation}
Consequently,  using $O(\log {N})$ measurements (sample voltage vectors) would be sufficient for  SGL  to learn an $N$-node graph for well preserving the original graph spectral  properties.

\subsection{Key Steps in the SGL Algorithm}\label{sec:overview}
To achieve good efficiency in graph learning that may involve a large number of nodes,  the proposed SGL algorithm   can iteratively identify and include the most influential edges  into the latest graph until no such edges can be found, through the following key steps. 

 \subsubsection{Step 1: Initial Graph Construction}
(\ref{optF2}) implies that by iteratively identifying and adding the most influential edges (with the highest sensitivities) into the latest graph,   the graph spectral embedding (or effective-resistance) distance will encode the   $\ell_2$ distances between the  original data vector space (averaged among $M$ measurements). To gain faster convergence of SGL,   sparsified $k$-nearest-neighbor (kNN) graphs \cite{malkov2018efficient} can be  leveraged as the initial graphs. However, choosing an optimal $k$ value (the number of nearest neighbors) for constructing the kNN graph can still be challenging for general graph learning tasks:  choosing a too large $k$   allows well approximating the global structure of the manifold for the original data points (voltage measurements), but will  result in a rather dense graph;  choosing a too small $k$ may lead to many small isolated graphs, which may  slow down the    iterations.  

Since circuit networks   are typically   very sparse (e.g. 2D or 3D meshes) in nature, the voltage or current measurements (vectors) usually lie near  low-dimensional manifolds, which allows finding a proper $k$ for our graph learning tasks. To achieve a good trade-off between   complexity and quality, in SGL    the initial graph will  be set up through the following steps: \textbf{(1)} Construct a connected kNN graph with a relatively small $k$ value (e.g. $5\le k \le 10$), which will suffice for approximating the global manifold corresponding to the original   measurement data; \textbf{(2)} Sparsify the kNN graph by extracting a  maximum spanning tree (MST)   that will serve as the initial graph. Later, SGL will   gradually  improve the graph by iteratively including the most influential off-tree edges selected from  the    kNN graph until convergence. 
  
 \subsubsection{Step 2: Spectral Graph Embedding}
Spectral graph embedding directly leverages the first few nontrivial eigenvectors  for mapping nodes onto low-dimensional space \cite{belkin2003laplacian}. The eigenvalue decomposition of  Laplacian matrix  is usually the computational bottleneck in spectral graph embedding, especially for large graphs. To achieve good scalability,  we can exploit fast multilevel eigensolvers  that allow computing the first few Laplacian eigenvectors  in nearly-linear time without loss of accuracy \cite{zhao2021towards}. 

 \subsubsection{Step 3: Influential Edge Identification}
Once the first few Laplacian eigenvectors are available, we can efficiently identify the most influential off-tree edges by looking at each candidate edge's sensitivity score  defined in (\ref{optF2}). In the proposed SGL approach,  each candidate off-tree edge  (in the kNN graph) will be sorted according to its edge sensitivity. Only a few most influential edges that have the largest sensitivities computed by (\ref{optF2}) will be included into the latest graph. Note that when   $r\ll N$, the following inequality  holds for any edge $(s, t)$: 
\begin{equation}\label{distBound}
 \|U_r^\top e_{s,t}\|^2_2=z_{s,t}^{emb} < \|U_N^\top e_{s,t}\|^2_2 \le R^\textit{eff}({s,t}),
\end{equation}
 implying that the sensitivities ($s_{s,t}$) computed using the first $r$ eigenvectors will always be   smaller than the actual ones. Obviously, using more eigenvectors for spectral   embedding   will lead to more accurate estimation of edge sensitivities. For typical circuit networks,    sensitivities computed using a small number (e.g. $r<5$) of   eigenvectors will suffice for identifying the most influential edges. 
 \subsubsection{Step 4: Convergence Checking}
  In this work, we propose to exploit the maximum edge sensitivities   computed by (\ref{optF2}) for checking the convergence of SGL iterations. If there exists no additional edge that has a sensitivity greater than a given threshold ($s_{max}\ge tol$), the SGL iterations can be terminated. It should be noted that choosing different  tolerance ($tol$) levels  will result in  graphs with  different densities. For example, choosing a smaller threshold   will require more edges   to be included so that the resultant spectral embedding distances on the learned graph can more precisely encode the distances between the original data points.  
   \subsubsection{Step 5:  Spectral Edge Scaling}
 Assume that $\sigma^2$ in (\ref{subspace}) approaches $+\infty$ and the normalized input  right-hand-side (current) vectors ($Y=[y_1, ..., y_M]$) corresponding to the $M$ voltage measurements ($X=[x_1, ..., x_M]$)   are orthogonal to the all-one vector.  Then for each original voltage vector $x_i$ and its corresponding current vector $y_i$ we have:
   \begin{equation}\label{scaling}
\|x_i\|^2_2=y_i^\top  (L^+_*)^2 y_i, ~~\textit{for}~~i=1,...,M.
\end{equation}
 Next, for each  $y_i$ we   compute the   voltage vector $\tilde x_i$  using the estimated Laplacian $L$   obtained via SGL iterations:
 \begin{equation}\label{scaling2}
L \tilde x_i= y_{i} => \|\tilde x_i\|^2_2=y_i^\top  (L^+)^2 y_i,~~\textit{for}~~i=1,...,M.
\end{equation}
To more precisely match the original graph spectral properties,  each edge weight can be adjusted as follows:
   \begin{equation}\label{scaling3}
w_{s,t}=\tilde w_{s,t} *\sqrt{\frac{1}{M}\sum\limits_{i=1}^M \frac{\|\tilde x_i\|^2_2}{\| x_i\|^2_2}},
\end{equation}
where $\tilde w_{s,t}$ denotes the initial edge weight obtained via  the previous SGL iterations. Since solving the ultra-sparse Laplacian matrix $L$ can be accomplished in nearly linear time \cite{miller:2010focs,zhiqiang:iccad17}, the proposed scaling scheme is highly efficient.
 
\subsection{Algorithm  Flow and Complexity }\label{main:complexity}
  The detailed SGL algorithm flow  has been shown in Algorithm \ref{alg:sgl}. All the aforementioned   steps in SGL can be accomplished in nearly-linear time by leveraging recent high-performance algorithms for   kNN graph construction \cite{malkov2018efficient},   spectral graph embedding   for influential edge identification \cite{zhao:dac19,zhao2021towards}, and fast Laplacian solver for edge scaling \cite{miller:2010focs,zhiqiang:iccad17}. Consequently, each SGL iteration can be accomplished in  nearly-linear time, whereas the state-of-the-art methods require at least $O(N^2)$ time \cite{dong2019learning}.  
  
\begin{algorithm}
 { \caption{The SGL Algorithm Flow} \label{alg:sgl}
\textbf{Input:} The voltage measurement matrix $X \in {\mathbb{R} ^{N\times M}}$, input current measurement matrix $Y \in {\mathbb{R} ^{N\times M}}$, $k$ for initial kNN graph construction, $r$ for constructing the  projection matrix in (\ref{subspace}), the maximum edge sensitivity tolerance ($tol$),  and the edge sampling ratio ($0 < \beta \le 1$). ~~~\textbf{Output:} The learned graph   $G=(V, E, w)$.\\

\begin{algorithmic}[1]
    \STATE {Construct a   kNN graph $G_{o}=(V, E_o, w_o)$ based on $X$.}
     \STATE {Extract an MST subgraph $T$ from $G_{o}$. }
    \STATE {Assign $G=T=(V, E, w)$ as the initial graph. }
     \WHILE{$s_{max} \ge tol$}
     \STATE{Compute the projection matrix $U_r$ with (\ref{subspace}) for the latest graph $G$.}
     \STATE{Sort  off-tree edges $(s,t)\in E_o\setminus E$  according to their sensitivities computed by $s_{s,t}=\frac{\partial F}{\partial w_{s,t} }$ using (\ref{optF2}).}
     \STATE{Include  an off-tree edge $(s, t)$ into $G$ if its  $s_{s, t}>tol$ and it has been ranked among the top   $\ceil{ N{\beta}}$   edges. }
       \STATE{Record the maximum edge sensitivity $s_{max}$. }
     \ENDWHILE
     \STATE{Do spectral edge  scaling  using $X$ and $Y$  via (\ref{scaling})-(\ref{scaling3});}
    \STATE {Return the learned graph $G$.}
\end{algorithmic}
}
\end{algorithm}
\vspace{-0pt}
\section{Experimental Results}\label{result_sec}
The proposed SGL algorithm has been implemented in Matlab. The test cases  in this paper have been selected from a great variety of matrices that have been used in circuit simulation and finite element analysis problems.  Since the prior state-of-the-art graph learning algorithms \cite{dong2019learning} have been developed based on a standard convex solver \cite{grant2009cvx}, the runtime would be excessively long (over many thousands of seconds) even for the smallest test case ($|V|=4,253$) reported in this paper. Therefore, we will only compare with the graph construction method based on the standard kNN algorithm. All of our experiments have been conducted using a single CPU core of a computing system with a $3.4$ GHz Quad-Core Intel Core i5 CPU and $24$ GB memory.
\vspace{-0pt}

\subsection{Experimental Setup}
To generate the voltage and current measurement samples, the following procedure has  been applied: (1) we first randomly generate $M$ current source vectors with each element sampled using a standard normal distribution; (2) each current vector will be normalized and orthogonal to the all-one vector; (3) $M$ voltage vector measurements will be obtained by solving the original graph Laplacian matrix  with the $M$ current vectors as the (right-hand-side) input vectors; (4)  the voltage and current vectors will be stored in matrices $X=[x_1,...,x_M]$ and $Y=[y_1,...,y_M] \in {\mathbb{R} ^{N\times M}}$, respectively, which will be used as the input data of the proposed SGL algorithm. By default, $M=50$ is used for generating the voltage and current measurements. We choose $k=5$ for constructing the kNN graph for all test cases. We set $r=5$ for constructing the  projection matrix in (\ref{subspace}). The edge sampling ratio  $\beta=10^{-3}$  has been used. The SGL iterations will be terminated if   $s_{max}<tol=10^{-12}$. When approximately computing the objective function value (\ref{opt2}), the first $50$ nonzero Laplacian eigenvalues are used. 

To clearly visualize each graph, the spectral graph drawing technique has been adopted \cite{koren2003spectral}: when creating the 2D   graph   layouts, each entry of the first two nontrivial Laplacian eigenvectors ($u_2, u_3$) corresponds to   the $x$ and $y$   coordinates  of each node (data point), respectively. We assign  the nodes with the same   color if they belong to the same node cluster determined by spectral graph clustering \cite{zhao2018nearly}. 
\vspace{-0pt}

\subsection{Comprehensive Results for Graph Learning} 
 

  \paragraph{Algorithm Convergence} As shown in Figure \ref{fig:distRes}, for the ``2D mesh" graph ($|V|=10,000$, $|E|=20,000$) learning task, SGL  requires about $40$ iterations to converge to $s_{max}\le 10^{-12}$ when starting from an initial  MST  of a 5NN graph.  
  \begin{figure}
     \centering\includegraphics[width=0.95\linewidth]{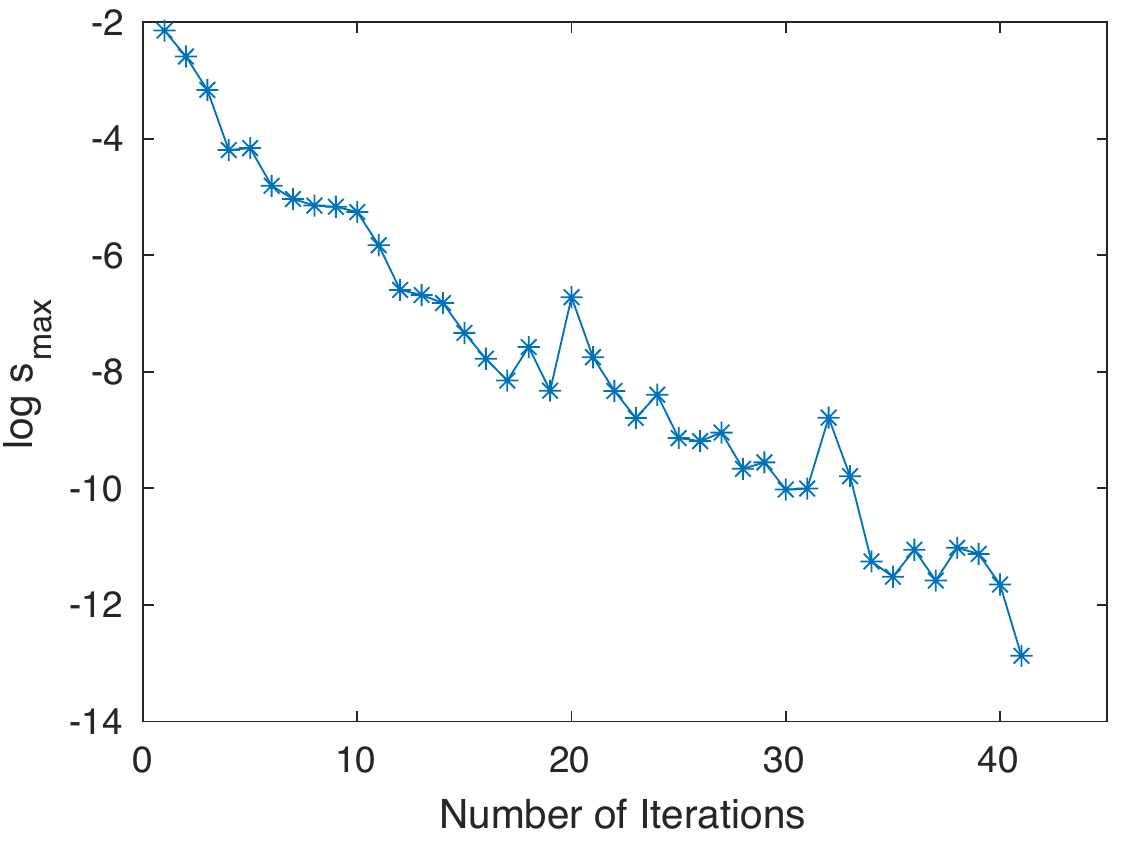}
  \caption{The decreasing maximum sensitivities (``2D mesh" graph) }\label{fig:distRes}
  \end{figure}
  
    \begin{figure}
   \centering\includegraphics[width=0.985\linewidth]{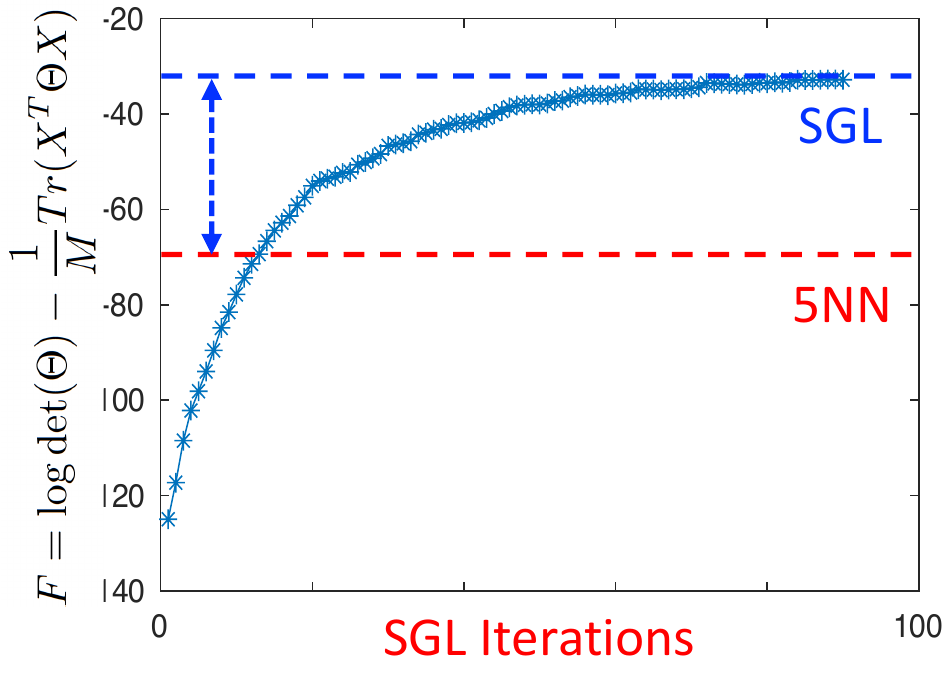}
  \caption{The objective function values (``fe\_4elt2" graph) }\label{fig:compKNN}
  \end{figure}

   \paragraph{Comparison with kNN Graph} As shown in Figure \ref{fig:compKNN}, for the ``fe\_4elt2" graph ($|V|=11,143$, $|E|=32,818$) learning task, SGL  converges in about  $90$ iterations when starting from an initial  MST of a 5NN graph. For the 5NN graph, we do the same edge   scaling     using (\ref{scaling})-(\ref{scaling3}). As shown in Figures \ref{fig:compKNN} and  \ref{fig:compKNN2}, the SGL-learned graph achieves a more optimal objective function value  and a much better  spectral approximation   than the 5NN graph. As observed, the SGL-learned graph has a density similar to a spanning tree, which is much sparser than the 5NN graph.
   \begin{figure}[!htb]
  \centering\includegraphics[width=0.995\linewidth]{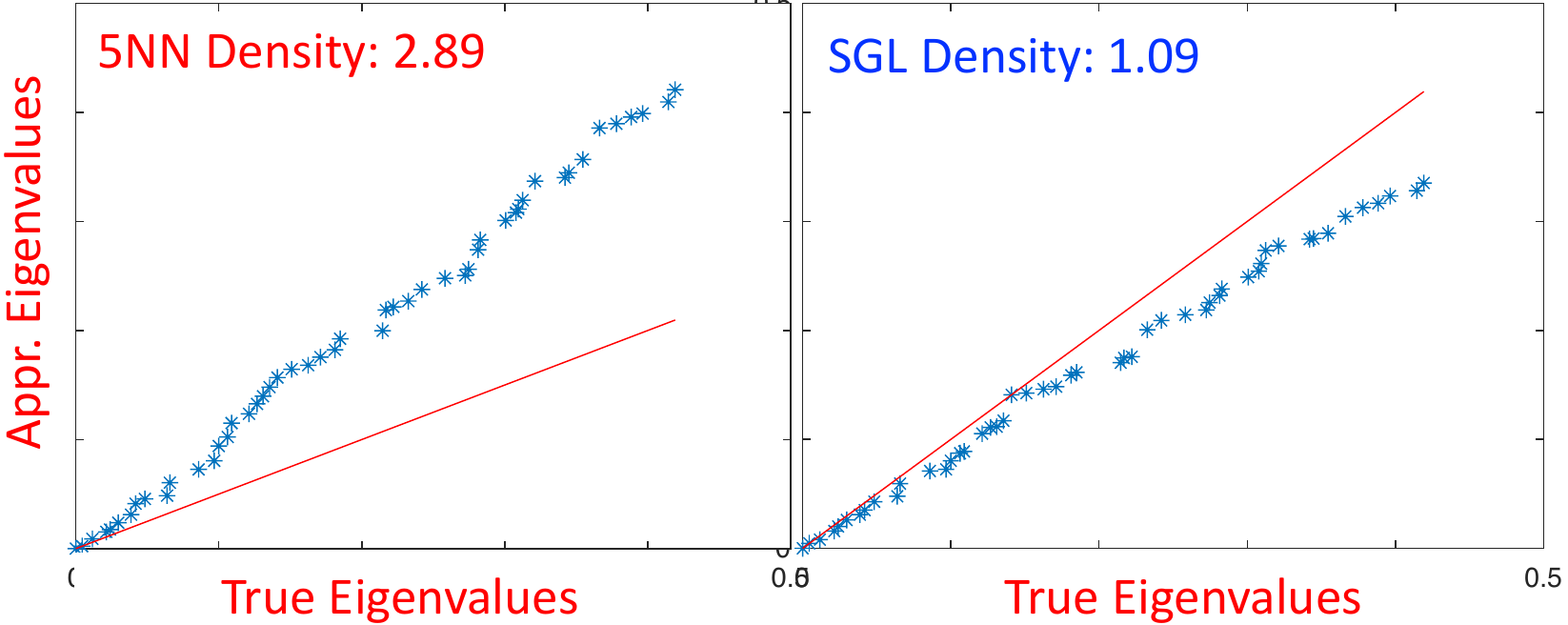}
  \caption{The comparison with a 5NN graph (``fe\_4elt2" graph)}\label{fig:compKNN2}
  \end{figure}
  
\paragraph{Learning Circuit Networks} As shown in Figures \ref{fig:airfoil} and \ref{fig:G2}  for the ``airfoil" ($|V|=4,253$, $|E|=12,289$), the ``crack" ($|V|=10,240$, $|E|=30,380$),  
and  the ``G2\_circuit" ($|V|=150,102$, $|E|=288,286$) graphs,   SGL   can consistently learn    ultra-sparse graphs which are slightly denser than  spanning trees while preserving the key graph spectral  properties. In Figure \ref{fig:resistance}, we observe highly correlated results when comparing the  effective resistances computed on the original  graphs with the  graphs learned by SGL.

\paragraph{Learning Reduced Networks} As shown in Figure \ref{fig:G2Redu}, by randomly choosing a small portion of node voltage measurements (without using any currents measurements in $Y$ matrix) for graph learning, SGL can learn  spectrally-similar graphs of much smaller sizes: when $20\%$ and  $10\%$  node voltage measurements are used for graph learning, $5\times$ and $10\times$ smaller resistor networks can be constructed, respectively, while preserving the key spectral (structural) properties of the original graph.

  \begin{figure}[!htb]
  \includegraphics[width=0.9995\linewidth]{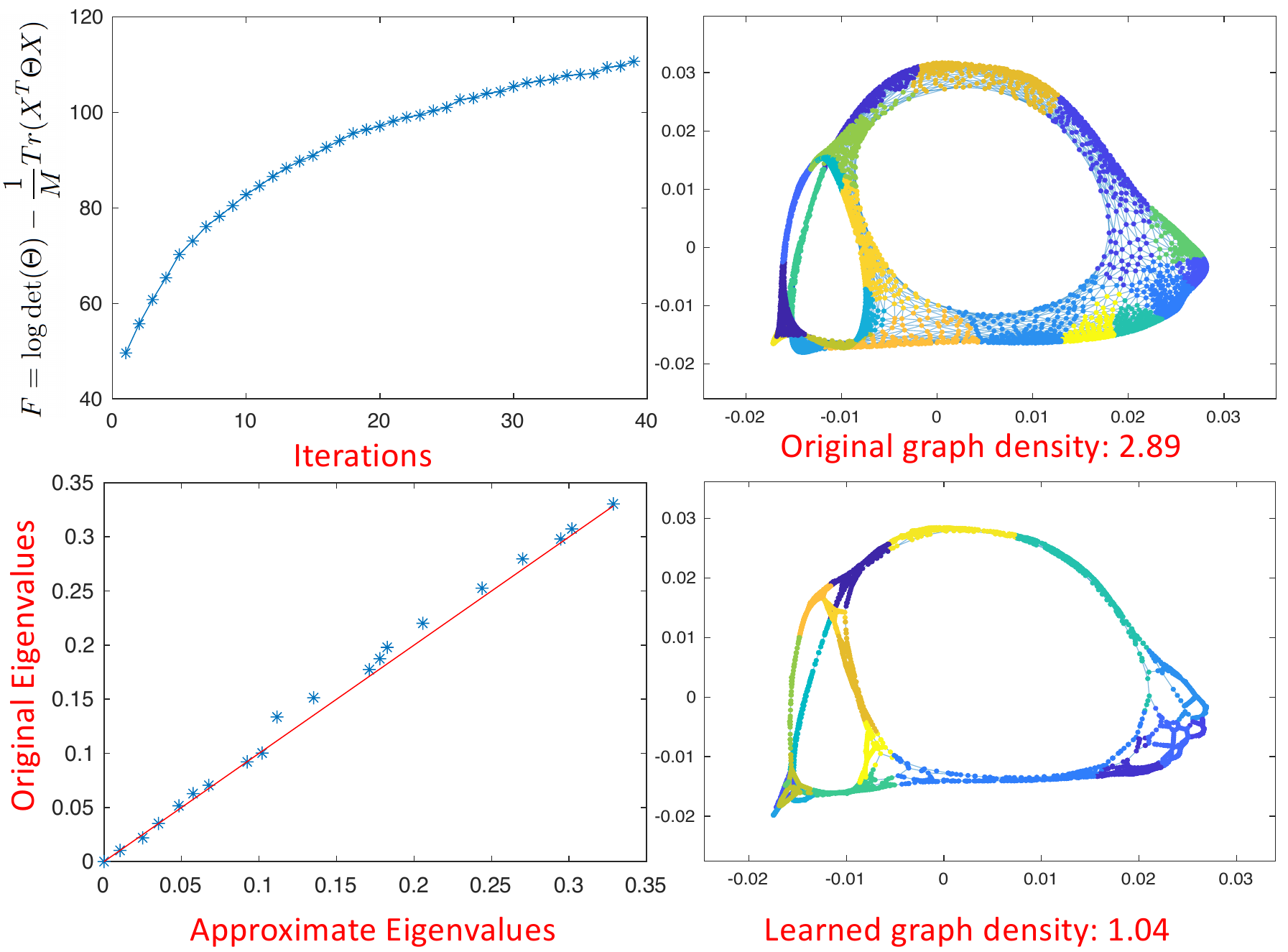}
  \caption{The results for learning the ``airfoil" graph}\label{fig:airfoil}
  \end{figure}
  
    \begin{figure}[!htb]
  \includegraphics[width=0.995\linewidth]{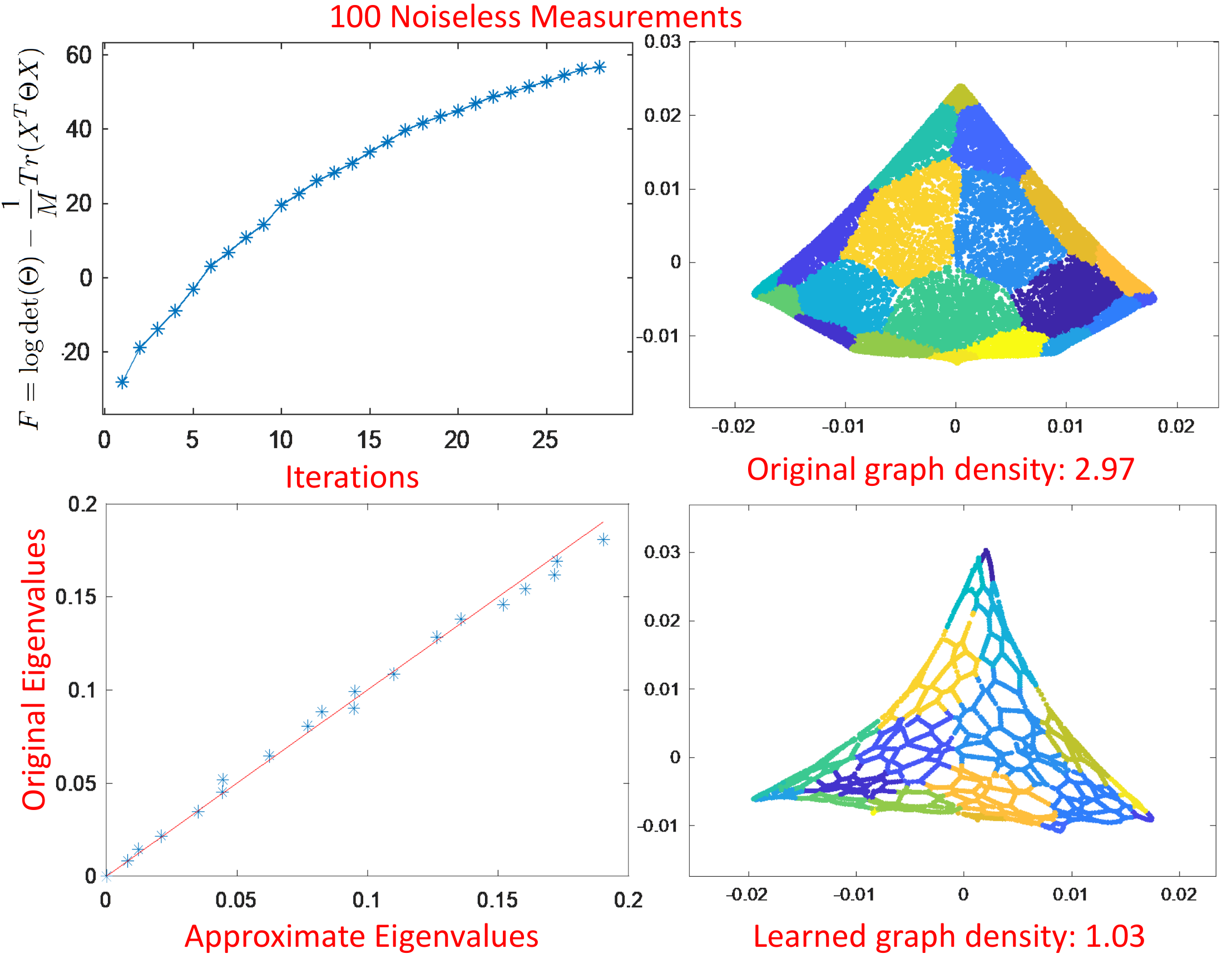}
  \caption{The results for learning the ``crack" graph}\label{fig:crack}
  \end{figure}
  
   \begin{figure}[!htb]
  \includegraphics[width=0.9995\linewidth]{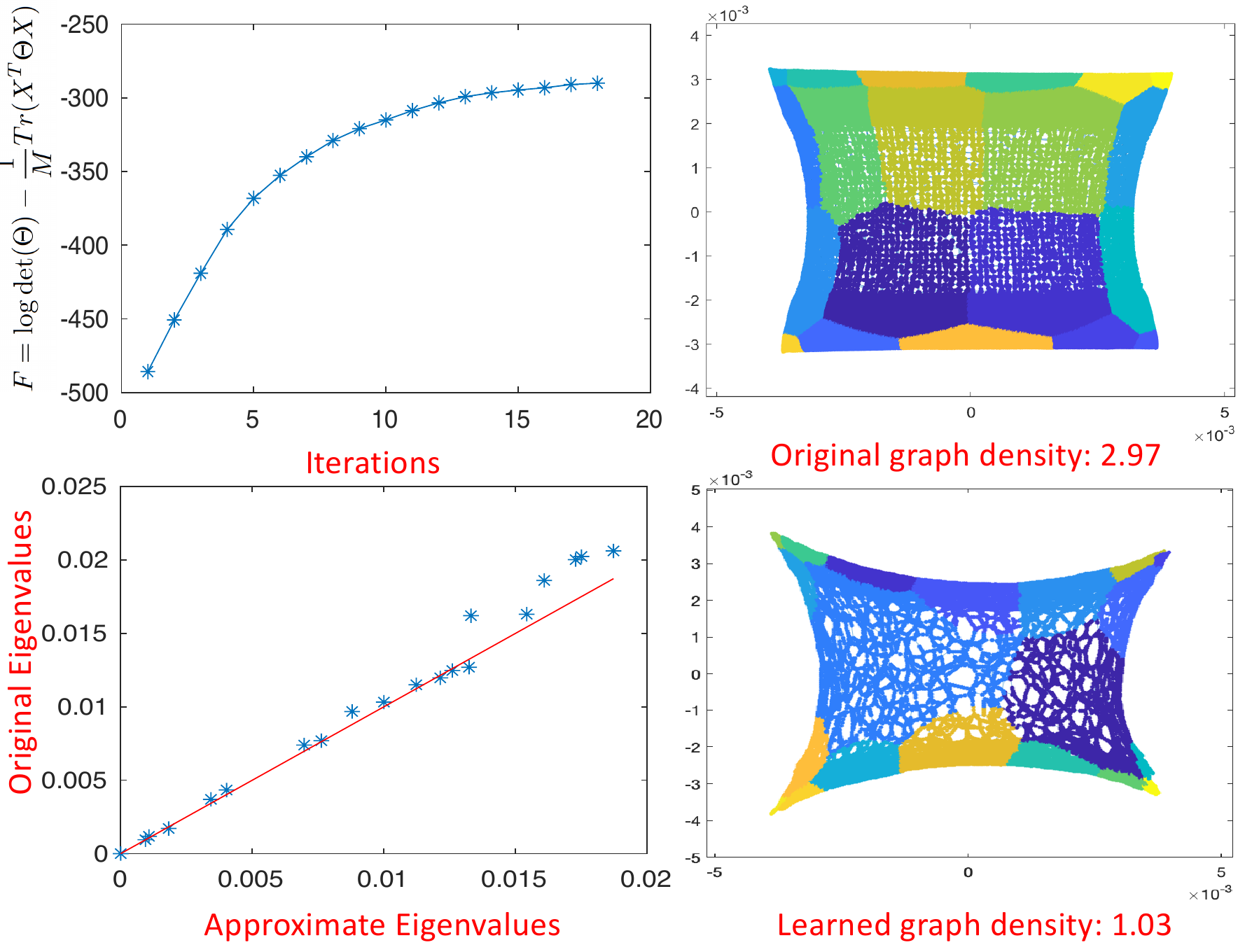}
  \caption{The results for learning the ``G2\_circuit" graph}\label{fig:G2}
  \end{figure}

    \begin{figure}[!htb]
  \includegraphics[width=0.9995\linewidth]{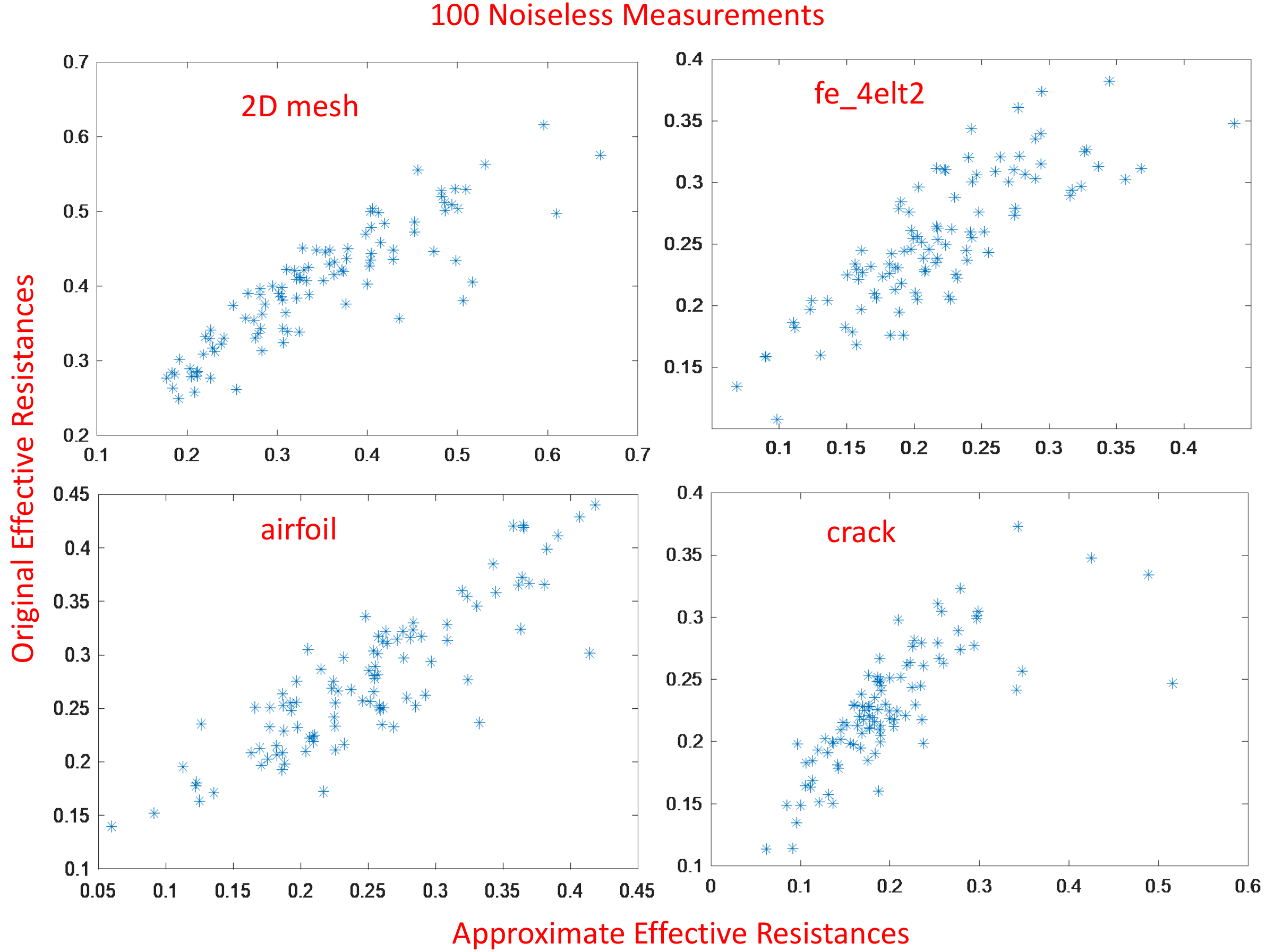}
  \caption{The effective resistances correlations (scatter plots) }\label{fig:resistance}
  \end{figure}
  
  \begin{figure}[!htb]
  \includegraphics[width=0.999\linewidth]{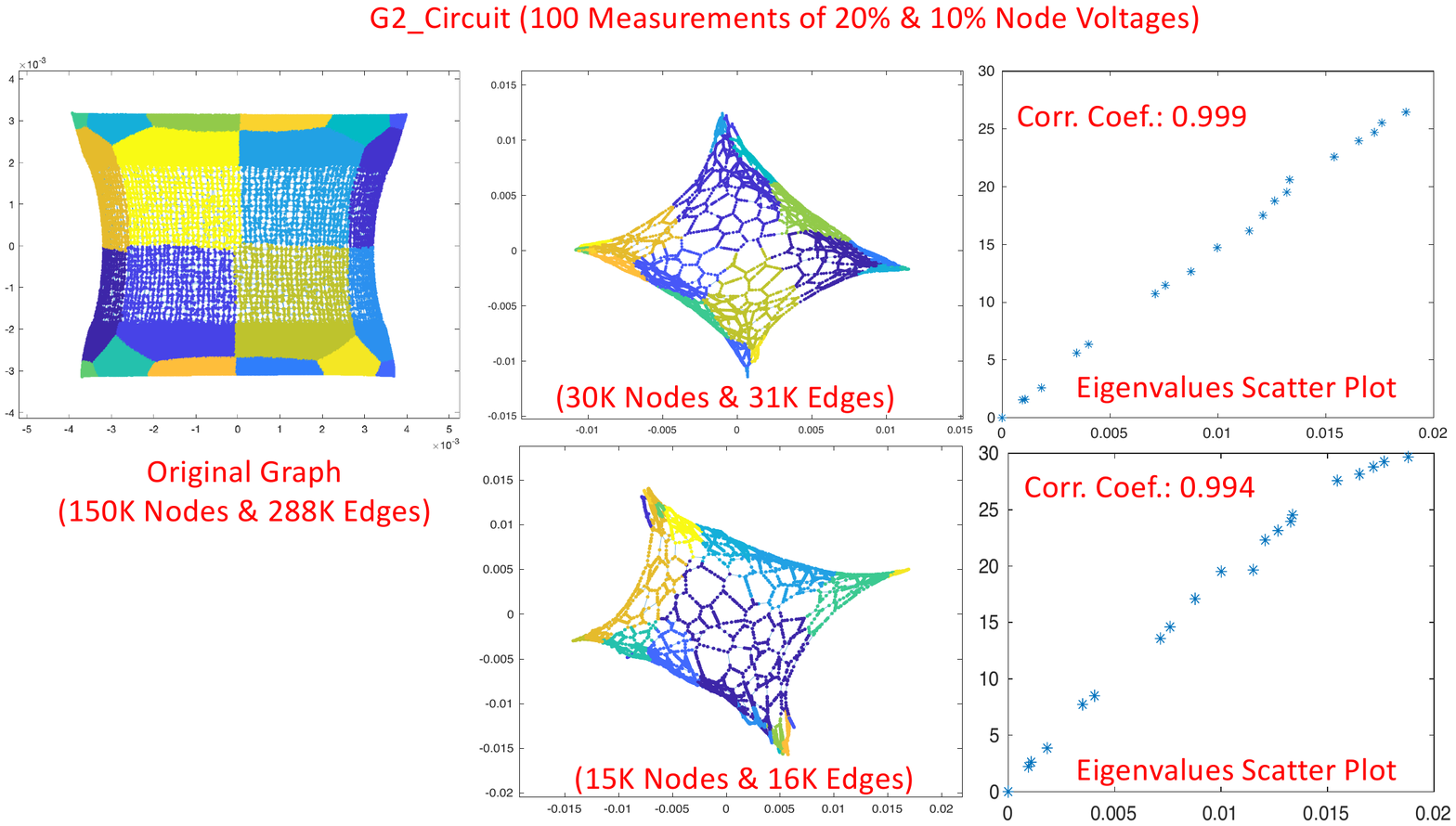}
  \caption{The reduced graphs learned by SGL (``G2\_circuit" )}\label{fig:G2Redu}
  \end{figure}
  
\paragraph{Learning with Noisy Measurements} We show the results of the ``2D mesh" graph learning  with noisy voltage measurements. For each SGL graph learning task, each input voltage measurement (vector) $\tilde x$ will be computed by: $\tilde x= x+ \zeta\| x\|_2\epsilon$, where $\epsilon$ denotes a normalized Gaussian noise vector, and $\zeta$ denotes the noise level. As shown in Figure \ref{fig:MeshNoise}, the increasing  noise levels will result in worse approximations of the original   spectral properties. It is also observed that even with a very significant noise level of $\zeta=0.5$, the graph learned by the proposed SGL algorithm can still preserve the first few Laplacian eigenvalues that are key to the graph structural (global) properties.

   \begin{figure}[!htb]
  \includegraphics[width=0.995\linewidth]{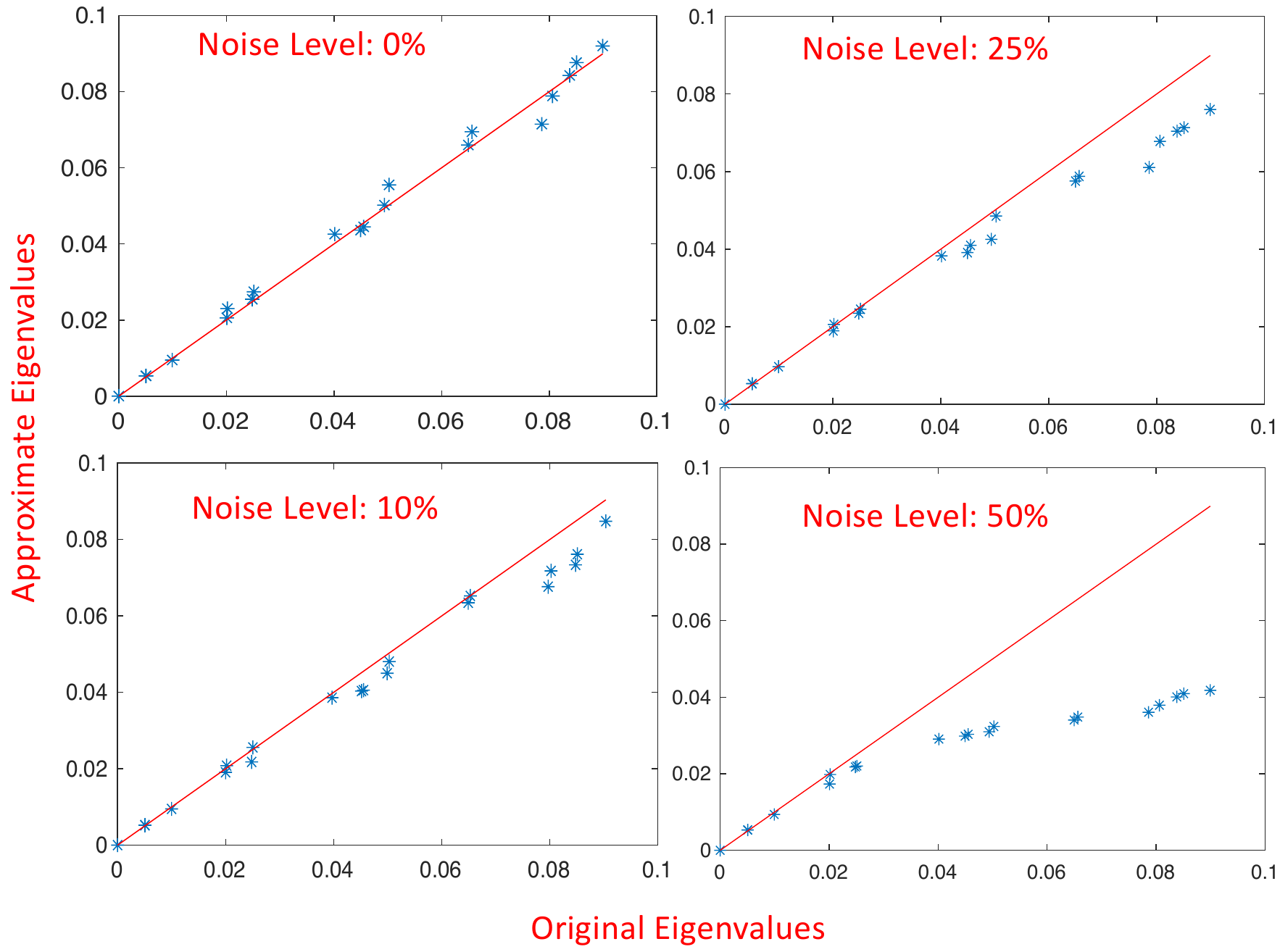}
  \caption{The graphs learned with noises (``2D mesh" graph)}\label{fig:MeshNoise}
  \end{figure}
  \begin{figure}[!htb]
  \includegraphics[width=0.9995\linewidth]{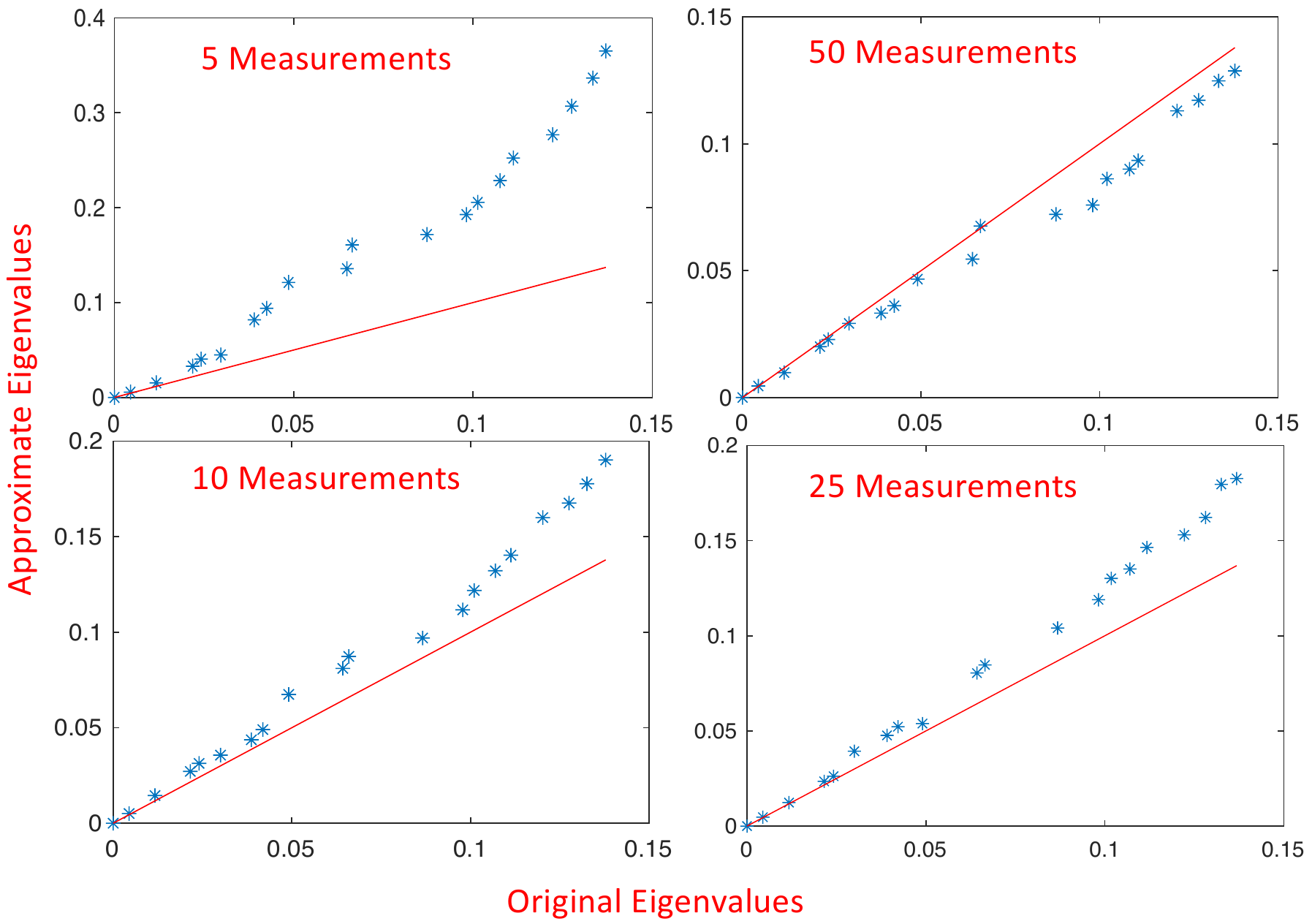}
  \caption{The effect of the number of measurements (``fe\_4elt2" graph) }\label{fig:sampleComplexity}
  \end{figure}
     \begin{figure}[!htb]
   \centering\includegraphics[width=0.9835\linewidth]{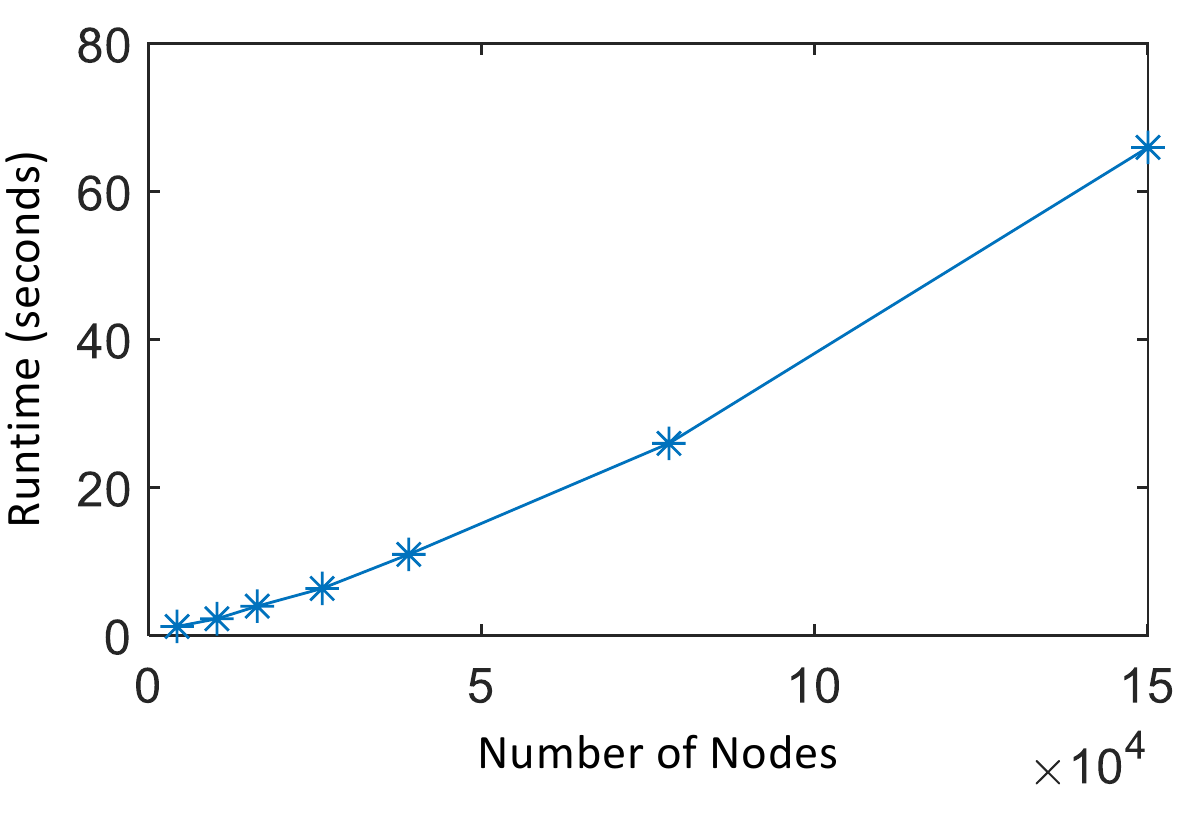}
    \vspace{-0pt}
  \caption{The runtime scalability of the SGL algorithm}\label{fig:runtime}
  \vspace{-5pt}
  \end{figure}
  \paragraph{Sample Complexity and Runtime Scalability} Figure \ref{fig:sampleComplexity} shows how the sample complexity (number of measurements)
may impact the graph learning quality. As observed, with increasing number of samples (measurements), substantially improved approximation of the  graph spectral    properties can be achieved. In the last, we show the runtime scalability of the proposed SGL algorithm. The runtime includes the total time of Step 2 to Step 5 but does not include the time for Step 1. Note that  modern kNN algorithms can achieve highly scalable runtime performance \cite{malkov2018efficient}.

  
\section{Conclusions}\label{conclusion}
 This work proposes a spectral algorithm (SGL) for learning resistor networks from linear voltage and current measurements.  
Our approach  iteratively identifies and includes the most influential edges to the latest graph. We show that the proposed graph learning approach is equivalent to solving the classical graphical Lasso problems with Laplacian-like precision matrices. A    unique feature of SGL is that the learned graphs will have  {  spectral embedding or effective-resistance  distances  encoding the similarities} between the original input data points (node voltages).  To  achieve high efficiency, SGL exploits a scalable spectral     embedding scheme    to allow  each iteration to be completed in $O(N \log N)$ time, whereas existing state-of-the-art methods   require at least $O(N^2)$ time for each iteration. We also provide  a sample complexity analysis  showing that  it is possible to   accurately recover a resistor network with only $O(\log N)$ voltage measurements (vectors).
  \vspace{-0pt}



\section{Acknowledgments}
This work is supported in part by  the National Science Foundation under Grants  CCF-2041519 (CAREER),   CCF-2021309 (SHF), and CCF-2011412 (SHF).

 \bibliographystyle{abbrv}
{
\bibliography{dac21,iclr20,feng}  
}

\end{document}